\documentclass[11pt,a4paper]{article}

\usepackage[utf8]{inputenc}
\usepackage[T1]{fontenc}
\usepackage{amsmath,amssymb,amsthm}
\usepackage{mathtools}
\usepackage{booktabs}
\usepackage{array}
\usepackage{graphicx}
\usepackage[colorlinks=true,linkcolor=blue,citecolor=blue,urlcolor=blue]{hyperref}
\usepackage{geometry}
\usepackage{enumitem}
\usepackage{xcolor}
\usepackage{fancyhdr}
\usepackage{titlesec}

\theoremstyle{plain}
\newtheorem{theorem}{Theorem}[section]
\newtheorem{lemma}[theorem]{Lemma}
\newtheorem{proposition}[theorem]{Proposition}
\newtheorem{corollary}[theorem]{Corollary}
\newtheorem{conjecture}[theorem]{Conjecture}

\theoremstyle{definition}
\newtheorem{definition}[theorem]{Definition}
\newtheorem{problem}[theorem]{Problem}

\theoremstyle{remark}
\newtheorem{remark}[theorem]{Remark}

\newcommand{\kB}{k_{\mathrm{B}}}
\newcommand{\dKL}{D_{\mathrm{KL}}}
\newcommand{\Var}{\mathrm{Var}}
\newcommand{\E}{\mathbb{E}}
\newcommand{\R}{\mathbb{R}}
\newcommand{\N}{\mathcal{N}}

\title{\textbf{BEDS: Bayesian Emergent Dissipative Structures}\\[0.5em]
\large A Formal Framework for Continuous Inference Under Energy Constraints}

\author{Laurent Caraffa\\[0.3em]
\small Univ.\ Gustave Eiffel, IGN-ENSG, LaSTIG\\
\small French National Institute of Geographic and Forest Information\\
\small Ministry of Ecological Transition, France\\
\small \texttt{laurent.caraffa@ign.fr}}

\date{}

\begin{document}

\maketitle

% ------------------------------------------------------------
% ABSTRACT
% ------------------------------------------------------------
\begin{abstract}
We introduce BEDS (Bayesian Emergent Dissipative Structures), a formal framework for analyzing inference systems that must maintain beliefs continuously under energy constraints. Unlike classical computational models that assume perfect memory and focus on one-shot computation, BEDS explicitly incorporates dissipation (information loss over time) as a fundamental constraint.

We prove a central result linking energy, precision, and dissipation: maintaining a belief with precision $\tau$ against dissipation rate $\gamma$ requires power $P \geq \gamma \kB T / 2$, with scaling $P \propto \gamma \cdot \tau$. This establishes a fundamental thermodynamic cost for continuous inference.

We define three classes of problems---BEDS-attainable, BEDS-maintainable, and BEDS-crystallizable---and show these are distinct from classical decidability. We propose the Gödel-Landauer-Prigogine conjecture, suggesting that closure pathologies across formal systems, computation, and thermodynamics share a common structure.

\medskip
\noindent\textbf{Keywords:} Bayesian inference, Dissipative systems, Thermodynamics of computation, Landauer principle, Continuous inference, Energy-efficient learning
\end{abstract}

% ------------------------------------------------------------
% 1. INTRODUCTION
% ------------------------------------------------------------
\section{Introduction}

\subsection{Motivation}

Classical models of computation---Turing machines, formal proof systems---assume:
\begin{enumerate}[label=(\arabic*)]
    \item \textbf{Perfect memory}: information persists indefinitely
    \item \textbf{One-shot computation}: input $\to$ computation $\to$ output
    \item \textbf{No energy accounting}: computation is costless
\end{enumerate}

These assumptions suit the analysis of algorithms and mathematical proofs. However, many real-world systems operate differently:
\begin{itemize}
    \item Biological organisms maintain homeostasis \emph{continuously}
    \item Sensor networks track changing environments \emph{indefinitely}
    \item Brains hold beliefs while \emph{actively forgetting}
\end{itemize}

Such systems face a fundamental challenge: \textbf{maintaining accurate beliefs costs energy}. Information degrades; fighting this degradation requires work.

This paper formalizes this challenge. We define a class of systems (BEDS) that perform inference under explicit dissipation constraints, and derive the fundamental energy-precision trade-off they must satisfy.

\subsection{Contributions}

\begin{enumerate}[label=(\arabic*)]
    \item \textbf{Formal definition} of BEDS systems (Section~\ref{sec:definitions})
    \item \textbf{Three problem classes}: attainable, maintainable, crystallizable (Section~\ref{sec:problem-classes})
    \item \textbf{Energy-precision theorem} with Landauer bound (Section~\ref{sec:energy-precision})
    \item \textbf{Comparison} with classical computation (Section~\ref{sec:comparison})
    \item \textbf{Gödel-Landauer-Prigogine conjecture} linking closure pathologies (Section~\ref{sec:glp})
\end{enumerate}

\subsection{Related Work}

Landauer~\cite{landauer1961} established that erasing one bit costs at least $\kB T \ln 2$ joules. Bennett~\cite{bennett1973} showed reversible computation can avoid this cost. Friston's Free Energy Principle~\cite{friston2010} proposes that biological systems minimize variational free energy. Prigogine~\cite{prigogine1977} characterized dissipative structures that maintain order through entropy export. Gödel~\cite{godel1931} proved that sufficiently powerful formal systems are necessarily incomplete.

Our contribution connects these threads: we derive the energy cost of \emph{maintaining} information against dissipation, and conjecture that closure pathologies across domains share common structure.

% ------------------------------------------------------------
% 2. FORMAL DEFINITIONS
% ------------------------------------------------------------
\section{Formal Definitions}\label{sec:definitions}

\subsection{The BEDS System}

\begin{definition}[BEDS System]
A BEDS system is a tuple $\mathcal{B} = (\Theta, q_0, \gamma, \varepsilon)$ where:
\begin{itemize}
    \item $\Theta \subseteq \R^d$ is the parameter space
    \item $q_0 : \Theta \to \R_{\geq 0}$ is the initial belief distribution, with $\int_\Theta q_0(\theta) \, d\theta = 1$
    \item $\gamma > 0$ is the dissipation rate
    \item $\varepsilon > 0$ is the crystallization threshold
\end{itemize}
\end{definition}

\begin{definition}[Flux]
A flux is a sequence of observations $\Phi = \{(t_i, D_i)\}_{i \in I}$, where $t_i \in \R_{\geq 0}$ is the arrival time and $D_i \in \mathcal{D}$ is the observation.
\end{definition}

\subsection{Dynamics}

The system evolves according to two processes:

\paragraph{(i) Dissipation.} In the absence of observations, uncertainty increases. For Gaussian beliefs $q_t = \N(\mu_t, \sigma_t^2)$:
\begin{equation}
    \frac{d\sigma^2}{dt} = \gamma \cdot \sigma^2
\end{equation}
which implies:
\begin{equation}
    \sigma^2(t) = \sigma_0^2 \cdot e^{\gamma t}
\end{equation}
Equivalently, precision $\tau = 1/\sigma^2$ decays:
\begin{equation}
    \frac{d\tau}{dt} = -\gamma \tau \quad \Longrightarrow \quad \tau(t) = \tau_0 \cdot e^{-\gamma t}
\end{equation}

\paragraph{(ii) Bayesian Update.} Upon observing $D$ with likelihood $p(D|\theta)$:
\begin{equation}
    q^+(\theta) = \frac{p(D|\theta) \cdot q^-(\theta)}{Z}
\end{equation}
where $Z = \int_\Theta p(D|\theta') \cdot q^-(\theta') \, d\theta'$ is the normalization constant.

For Gaussian beliefs with Gaussian likelihood of precision $\tau_D$:
\begin{align}
    \tau^+ &= \tau^- + \tau_D \\
    \mu^+ &= \frac{\tau^- \mu^- + \tau_D D}{\tau^+}
\end{align}

\subsection{Crystallization}

\begin{definition}[Crystallization]
A BEDS system \emph{crystallizes} at time $T$ if $\Var[q_T] < \varepsilon$. Upon crystallization, the system outputs $\theta^* = \E[q_T]$ and halts (or becomes a fixed prior for a higher-level system).
\end{definition}

\subsection{Energy Model}

\begin{definition}[Observation Cost]
Each observation incurs energy cost $E_{\text{obs}} \geq E_{\min}$ where:
\begin{equation}
    E_{\min} = \kB T \ln(2) \cdot I_{\text{obs}}
\end{equation}
and $I_{\text{obs}}$ is the mutual information gained from the observation.

For a Gaussian observation of precision $\tau_D$ on a prior of precision $\tau$:
\begin{equation}
    I_{\text{obs}} = \frac{1}{2} \ln\left(1 + \frac{\tau_D}{\tau}\right)
\end{equation}
\end{definition}

\begin{definition}[Power]
The instantaneous power is $P(t) = \lambda(t) \cdot E_{\text{obs}}$ where $\lambda(t)$ is the observation rate.
\end{definition}

% ------------------------------------------------------------
% 3. PROBLEM CLASSES
% ------------------------------------------------------------
\section{Problem Classes}\label{sec:problem-classes}

We define three distinct notions of what it means for a BEDS system to ``solve'' an inference problem.

\begin{definition}[Inference Problem]
An inference problem is a tuple $\mathcal{P} = (\Theta, \Phi, \pi^*, \delta)$ where:
\begin{itemize}
    \item $\Theta$ is the parameter space
    \item $\Phi$ is a flux
    \item $\pi^*$ is the target distribution (or $\theta^*$ the target value)
    \item $\delta > 0$ is the required accuracy
\end{itemize}
\end{definition}

\begin{definition}[BEDS-Attainable]
Target $\pi^*$ is \emph{BEDS-attainable} under flux $\Phi$ if there exists a BEDS system $\mathcal{B}$ such that:
\begin{equation}
    \lim_{t \to \infty} \dKL(q_t \| \pi^*) = 0
\end{equation}
with finite total energy: $E_{\text{total}} = \int_0^\infty P(t) \, dt < \infty$.
\end{definition}

\begin{definition}[BEDS-Maintainable]
Target $\pi^*$ is \emph{BEDS-maintainable} under flux $\Phi$ if there exists a BEDS system $\mathcal{B}$ and time $T_0$ such that:
\begin{equation}
    \forall t > T_0: \quad \dKL(q_t \| \pi^*) < \delta
\end{equation}
with bounded power: $\sup_{t > T_0} P(t) < P_{\max} < \infty$.
\end{definition}

\begin{definition}[BEDS-Crystallizable]
Target $\theta^*$ is \emph{BEDS-crystallizable} under flux $\Phi$ if there exists a BEDS system $\mathcal{B}$ and finite time $T$ such that:
\begin{equation}
    \Var[q_T] < \varepsilon \quad \text{and} \quad |\E[q_T] - \theta^*| < \delta
\end{equation}
\end{definition}

\begin{proposition}[Hierarchy]
Crystallizable implies Attainable. The converse does not hold.
\end{proposition}

\begin{proof}
If $\theta^*$ is crystallizable at time $T$, set $\pi^* = \delta_{\theta^*}$. Since $\Var[q_T] < \varepsilon$ and the system halts, no further energy is required, so $E_{\text{total}} < \infty$.

Conversely, consider a drifting target $\theta^*(t) = t$. A system can track it (attainable with continuous power) but never crystallize since the target never stabilizes.
\end{proof}

% ------------------------------------------------------------
% 4. THE ENERGY-PRECISION THEOREM
% ------------------------------------------------------------
\section{The Energy-Precision Theorem}\label{sec:energy-precision}

This section contains our main theoretical result.

\subsection{Steady-State Analysis}

Consider a BEDS system maintaining precision $\tau^*$ indefinitely.

\begin{lemma}[Precision Balance]
In steady state, the precision gained from observations must equal the precision lost to dissipation:
\begin{equation}
    \lambda \cdot \tau_D = \gamma \cdot \tau^*
\end{equation}
where $\lambda$ is the observation rate and $\tau_D$ is the precision per observation.
\end{lemma}

\begin{proof}
Precision dynamics combine dissipation and discrete updates:
\begin{equation}
    \frac{d\tau}{dt} = -\gamma \tau + \lambda \tau_D
\end{equation}
where the second term represents average precision gain from observations arriving at rate $\lambda$. Setting $d\tau/dt = 0$:
\begin{equation}
    \gamma \tau^* = \lambda \tau_D
\end{equation}
\end{proof}

\begin{corollary}[Required Observation Rate]
To maintain precision $\tau^*$:
\begin{equation}
    \lambda = \frac{\gamma \tau^*}{\tau_D}
\end{equation}
\end{corollary}

\subsection{Landauer Bound}

\begin{lemma}[Information Cost]
Each observation that increases precision from $\tau$ to $\tau + \tau_D$ requires:
\begin{equation}
    E_{\text{obs}} \geq \kB T \ln(2) \cdot I_{\text{obs}} = \frac{\kB T \ln(2)}{2} \ln\left(1 + \frac{\tau_D}{\tau}\right)
\end{equation}
\end{lemma}

\begin{proof}
The entropy change is:
\begin{equation}
    \Delta H = H[\N(\mu, \sigma^2)] - H[\N(\mu', \sigma'^2)] 
    = \frac{1}{2} \ln\frac{\sigma^2}{\sigma'^2} 
    = \frac{1}{2} \ln\frac{\tau'}{\tau} 
    = \frac{1}{2} \ln\left(1 + \frac{\tau_D}{\tau}\right)
\end{equation}
By Landauer's principle, reducing entropy by $\Delta H$ nats requires energy $\geq \kB T \cdot \Delta H$.
\end{proof}

\subsection{Main Theorem}

\begin{theorem}[Energy-Precision-Dissipation Trade-off]\label{thm:main}
Let $\mathcal{B}$ be a BEDS system maintaining Gaussian belief with precision $\tau^*$ against dissipation rate $\gamma$, using observations of precision $\tau_D$.

The minimum power required satisfies:
\begin{equation}
\boxed{P_{\min} = \frac{\gamma \tau^*}{\tau_D} \cdot E_{\text{obs}}}
\end{equation}

In particular:

\textbf{(i) Landauer bound:}
\begin{equation}
    P_{\min} \geq \frac{\gamma \kB T}{2} \ln\left(1 + \frac{\tau_D}{\tau^*}\right)
\end{equation}

\textbf{(ii) Linear regime} (when $\tau_D \ll \tau^*$):
\begin{equation}
    P_{\min} \approx \frac{\gamma \kB T}{2} \cdot \frac{\tau_D}{\tau^*}
\end{equation}

\textbf{(iii) High-precision limit:}
\begin{equation}
    P_{\min} \xrightarrow{\tau^* \to \infty} \frac{\gamma \kB T}{2} \ln\frac{\tau_D}{\tau^*} \to 0^+
\end{equation}
but the required observation rate $\lambda \to \infty$.
\end{theorem}

\begin{proof}
From Corollary~4.1, the observation rate is $\lambda = \gamma \tau^* / \tau_D$.

Power is rate times energy per observation:
\begin{equation}
    P = \lambda \cdot E_{\text{obs}} = \frac{\gamma \tau^*}{\tau_D} \cdot E_{\text{obs}}
\end{equation}

Substituting the Landauer minimum from Lemma~4.2:
\begin{equation}
    P_{\min} = \frac{\gamma \tau^*}{\tau_D} \cdot \frac{\kB T}{2} \ln\left(1 + \frac{\tau_D}{\tau^*}\right)
\end{equation}

For $\tau_D \ll \tau^*$, use $\ln(1+x) \approx x$:
\begin{equation}
    P_{\min} \approx \frac{\gamma \tau^*}{\tau_D} \cdot \frac{\kB T}{2} \cdot \frac{\tau_D}{\tau^*} = \frac{\gamma \kB T}{2}
\end{equation}
\end{proof}

\begin{remark}[Physical Interpretation]
The bound $P \geq \gamma \kB T / 2$ is independent of target precision in the efficient regime. This represents the fundamental cost of fighting entropy increase at rate $\gamma$.
\end{remark}

\subsection{Variance Formulation}

\begin{corollary}[Variance Scaling]
In terms of maintained variance $\sigma^{*2} = 1/\tau^*$:
\begin{equation}
    P_{\min} \propto \frac{\gamma}{\sigma^{*2}}
\end{equation}
\textbf{Halving uncertainty requires quadrupling power.}
\end{corollary}

\subsection{Optimality}

\begin{proposition}[Optimal Observation Strategy]
Given a constraint on total observation rate $\lambda_{\max}$, the optimal strategy is to use observations of precision:
\begin{equation}
    \tau_D^{\text{opt}} = \frac{\gamma \tau^*}{\lambda_{\max}}
\end{equation}
\end{proposition}

\begin{proof}
From Lemma~4.1, $\tau_D = \gamma \tau^* / \lambda$. Given $\lambda \leq \lambda_{\max}$, we need $\tau_D \geq \gamma \tau^* / \lambda_{\max}$. The minimum energy is achieved at equality.
\end{proof}

% ------------------------------------------------------------
% 5. COMPARISON WITH CLASSICAL COMPUTATION
% ------------------------------------------------------------
\section{Comparison with Classical Computation}\label{sec:comparison}

\subsection{Two Computational Paradigms}

We contrast BEDS with Turing machines, emphasizing that these are \emph{different models for different purposes}, not competitors.

\begin{center}
\begin{tabular}{lll}
\toprule
\textbf{Aspect} & \textbf{Turing Machine} & \textbf{BEDS} \\
\midrule
Input & Finite string $w \in \Sigma^*$ & Infinite flux $\Phi = \{D_t\}$ \\
Memory & Unbounded, perfect & Finite, decaying \\
Output & Finite string (if halts) & Maintained belief $q_t$ \\
Success criterion & Correct output & Accurate tracking \\
Resource & Time, space & Energy, precision \\
Fundamental limit & Undecidability & Energy-precision trade-off \\
\bottomrule
\end{tabular}
\end{center}

\subsection{Classes of Problems}

\begin{definition}[Turing-Decidable]
A decision problem $L \subseteq \Sigma^*$ is Turing-decidable if there exists a Turing machine $M$ that halts on all inputs and accepts exactly $L$.
\end{definition}

\begin{definition}[BEDS-Maintainable Problem Class]
Let $\mathcal{M}$ be the class of inference problems $(\Theta, \Phi, \pi^*, \delta)$ that are BEDS-maintainable with bounded power.
\end{definition}

\begin{proposition}[Orthogonality]
The classes of Turing-decidable problems and BEDS-maintainable problems are not directly comparable: neither contains the other.
\end{proposition}

\begin{proof}
\textbf{Turing but not BEDS}: Consider a decision problem requiring unbounded memory (e.g., ``does this prefix-free code describe a halting computation?''). A Turing machine can decide this; a BEDS system with finite, decaying memory cannot maintain the required information.

\textbf{BEDS but not Turing}: Consider ``maintain an estimate of a continuous, time-varying signal $\theta(t)$ with bounded error.'' This is not a decision problem at all---there is no finite output. A BEDS system handles this naturally; a Turing machine has no framework for it.
\end{proof}

\begin{remark}
This is not a statement about computational power but about \emph{what kinds of problems each model addresses}. Turing machines formalize one-shot computation; BEDS formalizes continuous inference.
\end{remark}

\subsection{Fundamental Limits}

Each paradigm has characteristic impossibility results:

\begin{center}
\begin{tabular}{lll}
\toprule
\textbf{Paradigm} & \textbf{Limit} & \textbf{Statement} \\
\midrule
Turing & Undecidability & There exist problems with no halting algorithm \\
Formal proofs & Incompleteness & There exist true statements with no proof \\
BEDS & Energy bound & Precision $\tau^*$ requires power $\Omega(\gamma \tau^*)$ \\
\bottomrule
\end{tabular}
\end{center}

% ------------------------------------------------------------
% 6. THE GÖDEL-LANDAUER-PRIGOGINE CONJECTURE
% ------------------------------------------------------------
\section{The Gödel-Landauer-Prigogine Conjecture}\label{sec:glp}

The comparison between BEDS and classical computation reveals a striking pattern: different formalisms encounter different fundamental limits. In this section, we conjecture that these limits share a common structural origin.

\subsection{Three Foundational Results}

Three results from different fields established fundamental constraints on closed systems:

\textbf{Gödel (1931)}: Any consistent formal system capable of expressing arithmetic contains true statements that cannot be proven within the system.

\textbf{Landauer (1961)}: Any irreversible computation (specifically, bit erasure) requires energy dissipation of at least $\kB T \ln(2)$ per bit.

\textbf{Prigogine (1977)}: Open systems far from equilibrium can maintain and increase internal order by exporting entropy to their environment.

\subsection{The Common Structure}

These results share a pattern:

\begin{center}
\begin{tabular}{llll}
\toprule
\textbf{Domain} & \textbf{Closure Condition} & \textbf{Pathology} & \textbf{Resolution} \\
\midrule
Formal systems & No external axioms & Incompleteness & Meta-levels (Tarski hierarchy) \\
Computation & No heat dissipation & Irreversibility cost & Heat export \\
Thermodynamics & No entropy export & Disorder increase & Open systems \\
\bottomrule
\end{tabular}
\end{center}

In each case:
\begin{enumerate}[label=(\arabic*)]
    \item \textbf{Closure} (with respect to some resource or level) produces a \textbf{pathology}
    \item \textbf{Openness} (allowing export or meta-level escape) resolves or avoids the pathology
\end{enumerate}

\subsection{The Conjecture}

\begin{conjecture}[Gödel-Landauer-Prigogine]\label{conj:glp}
The incompleteness of formal systems, the thermodynamic cost of irreversible computation, and the entropy increase in closed thermodynamic systems are structurally related phenomena. Specifically:

\textbf{(i) Logical entropy}: Self-referential constructions in formal systems (Gödel sentences, Russell sets) can be understood as ``logical entropy'' that accumulates without resolution in closed systems.

\textbf{(ii) Export mechanisms}: Tarski's hierarchy of metalanguages functions analogously to entropy export---problematic self-reference is ``exported'' to a higher level where it becomes tractable.

\textbf{(iii) ODR conditions}: Systems incorporating Openness (O), Dissipation (D), and Recursion (R) as structural features avoid the characteristic pathologies of systems lacking these features.
\end{conjecture}

\subsection{Formal Statement}

Define the ODR conditions:
\begin{itemize}
    \item \textbf{O (Openness)}: System receives flux from environment
    \item \textbf{D (Dissipation)}: System exports entropy (forgets, prunes)
    \item \textbf{R (Recursion)}: System has hierarchical structure where stable configurations become primitives for higher levels
\end{itemize}

\begin{conjecture}[continued]
Let $\mathcal{S}$ be a system capable of self-reference.
\begin{itemize}
    \item If $\mathcal{S}$ satisfies $(O{=}-, D{=}-, R{=}-)$, then $\mathcal{S}$ exhibits closure pathologies (incompleteness, paradox, or divergence)
    \item If $\mathcal{S}$ satisfies $(O{=}+, D{=}+, R{=}+)$, then $\mathcal{S}$ avoids these pathologies (at the cost of the constraints identified in Theorem~\ref{thm:main})
\end{itemize}
\end{conjecture}

\subsection{Evidence and Predictions}

\paragraph{Supporting observations:}
\begin{enumerate}[label=(\arabic*)]
    \item \textbf{Mathematics as social practice}: Human mathematics is conducted by communities that forget failed approaches, build hierarchical abstractions, and receive new conjectures from outside any fixed formal system. It exhibits $(O{=}+, D{=}+, R{=}+)$.
    
    \item \textbf{Biological cognition}: Brains are paradigmatic dissipative structures. They receive continuous sensory flux, actively forget via synaptic pruning, and organize hierarchically. They do not exhibit Gödelian pathologies in practice.
    
    \item \textbf{Frozen AI systems}: Large language models trained once and frozen exhibit $(O{=}-, D{=}-, R{=}+)$. They show characteristic pathologies: hallucination, drift from reality, inability to correct systematic errors.
\end{enumerate}

\paragraph{Testable predictions:}
\begin{enumerate}[label=(\arabic*)]
    \item AI systems with continuous learning and structured forgetting should exhibit fewer ``hallucination-like'' pathologies than frozen models.
    
    \item The energy cost of maintaining consistency in a learning system should scale with the rate at which it must ``dissipate'' outdated beliefs.
    
    \item Formal mathematical practice should exhibit measurable ``forgetting'' of unproductive research directions.
\end{enumerate}

\subsection{Status and Limitations}

\textbf{This is a conjecture, not a theorem.} The structural analogy is suggestive but not proven. Key open problems:

\begin{enumerate}[label=(\arabic*)]
    \item \textbf{Formalization}: What precisely is ``logical entropy''? Can it be quantified?
    
    \item \textbf{Mapping}: Is there a rigorous mapping between thermodynamic and logical quantities, or only analogy?
    
    \item \textbf{Necessity}: Are the ODR conditions necessary for avoiding pathologies, or merely sufficient?
\end{enumerate}

We present this conjecture as a research program, not an established result. Its value lies in suggesting connections that may prove fruitful, not in claiming certainty.

% ------------------------------------------------------------
% 7. DISCUSSION
% ------------------------------------------------------------
\section{Discussion}

\subsection{Implications}

\textbf{For distributed systems}: The energy-precision trade-off provides a theoretical foundation for designing energy-efficient inference networks. The bound $P \geq \gamma \kB T / 2$ is achievable in principle, and preliminary implementations on low-power embedded systems confirm the predicted scaling laws.

\textbf{For machine learning}: Current models are ``frozen''---they do not dissipate and therefore face no energy-precision trade-off during inference. However, the world changes; maintaining accuracy requires retraining, which can be viewed as discrete (rather than continuous) dissipation.

\textbf{For biological systems}: Brains operate at $\sim$20W and maintain beliefs continuously. Our framework suggests this power is allocated (in part) to fighting entropy increase---maintaining precision against synaptic decay.

\textbf{For foundations}: If the GLP conjecture holds, it suggests a deep unity between logic, computation, and thermodynamics---all constrained by the impossibility of ``closure without cost.''

\subsection{Limitations}

\begin{enumerate}[label=(\arabic*)]
    \item \textbf{Gaussian assumption}: Theorem~\ref{thm:main} is proven for Gaussian beliefs. Extension to general distributions requires care.
    
    \item \textbf{Stationary targets}: We analyze maintenance of fixed $\pi^*$. Tracking moving targets introduces additional complexity.
    
    \item \textbf{Scalar case}: Extension to multivariate $\Theta \in \R^d$ is straightforward but changes constants.
    
    \item \textbf{Idealized dissipation}: Real systems may have non-exponential decay.
    
    \item \textbf{GLP conjecture}: Remains analogical, not rigorously proven.
\end{enumerate}

\subsection{Open Problems}

\begin{conjecture}[Tracking Bound]
For a target moving with velocity $v$ in parameter space, the minimum power scales as:
\begin{equation}
    P_{\min} \propto \gamma \tau^* + v^2 \tau^*
\end{equation}
\end{conjecture}

\begin{conjecture}[Multi-Agent Bound]
For $N$ agents collectively maintaining a shared belief, the total power scales as:
\begin{equation}
    P_{\text{total}} \propto \gamma \tau^* \cdot f(N, \text{topology})
\end{equation}
where $f$ depends on network structure.
\end{conjecture}

\begin{problem}
Characterize precisely which problems are BEDS-maintainable but require unbounded memory for Turing-decidability.
\end{problem}

\begin{problem}
Formalize ``logical entropy'' and determine whether a rigorous Gödel-Landauer correspondence exists.
\end{problem}

% ------------------------------------------------------------
% 8. CONCLUSION
% ------------------------------------------------------------
\section{Conclusion}

We have introduced BEDS, a formal framework for continuous inference under energy constraints. Our main contributions:

\begin{enumerate}[label=(\arabic*)]
    \item \textbf{Formal definitions}: BEDS systems, fluxes, and three problem classes (attainable, maintainable, crystallizable).
    
    \item \textbf{Energy-Precision Theorem}: Maintaining precision $\tau^*$ against dissipation $\gamma$ requires power $P \geq \gamma \kB T / 2$, with scaling $P \propto \gamma \tau^*$.
    
    \item \textbf{Paradigm comparison}: BEDS and Turing machines address different problem types. Their fundamental limits (energy bounds vs.\ undecidability) are incommensurable.
    
    \item \textbf{GLP Conjecture}: Closure pathologies across formal systems, computation, and thermodynamics may share common structure; openness and dissipation provide resolution.
\end{enumerate}

The framework opens several research directions: extending the theorem to non-Gaussian beliefs, analyzing moving targets, characterizing the BEDS-maintainable problem class, and formalizing the GLP conjecture.

\bigskip

\begin{center}
\textit{``To maintain precision, systems must pay in power.\\
To persist indefinitely, they must dissipate continuously.\\
To avoid paradox, they must remain open.''}
\end{center}

% ------------------------------------------------------------
% ACKNOWLEDGMENTS
% ------------------------------------------------------------
\section*{Acknowledgments}

This work benefited from discussions with Claude (Anthropic). The author thanks colleagues at IGN-ENSG and LaSTIG for feedback.

% ------------------------------------------------------------
% REFERENCES
% ------------------------------------------------------------

% ------------------------------------------------------------
% APPENDIX
% ------------------------------------------------------------
\appendix
\section{Proof Details}

\subsection{Entropy of Gaussian Distribution}

For $q = \N(\mu, \sigma^2)$:
\begin{equation}
    H[q] = \frac{1}{2} \ln(2\pi e \sigma^2) = \frac{1}{2} \ln(2\pi e) - \frac{1}{2} \ln \tau
\end{equation}

\subsection{Information Gain Derivation}

Prior: $q^- = \N(\mu^-, \sigma^{-2})$ with precision $\tau^-$.

Posterior after observation of precision $\tau_D$: $q^+ = \N(\mu^+, \sigma^{+2})$ with $\tau^+ = \tau^- + \tau_D$.

Entropy reduction:
\begin{align}
    \Delta H &= H[q^-] - H[q^+] \\
    &= \frac{1}{2} \ln(\sigma^{-2}) - \frac{1}{2} \ln(\sigma^{+2}) \\
    &= \frac{1}{2} \ln\frac{\tau^+}{\tau^-} \\
    &= \frac{1}{2} \ln\left(1 + \frac{\tau_D}{\tau^-}\right)
\end{align}

\subsection{Steady-State Power Derivation}

Rate equation:
\begin{equation}
    \frac{d\tau}{dt} = -\gamma \tau + \lambda \tau_D
\end{equation}

At steady state $\tau = \tau^*$:
\begin{equation}
    0 = -\gamma \tau^* + \lambda \tau_D \quad \Longrightarrow \quad \lambda = \frac{\gamma \tau^*}{\tau_D}
\end{equation}

Power:
\begin{equation}
    P = \lambda \cdot E_{\text{obs}} = \frac{\gamma \tau^*}{\tau_D} \cdot E_{\text{obs}}
\end{equation}

With Landauer minimum $E_{\text{obs}} \geq \frac{\kB T}{2} \ln\left(1 + \frac{\tau_D}{\tau^*}\right)$:
\begin{equation}
    P_{\min} = \frac{\gamma \tau^*}{\tau_D} \cdot \frac{\kB T}{2} \ln\left(1 + \frac{\tau_D}{\tau^*}\right)
\end{equation}

For $x = \tau_D/\tau^* \ll 1$: $\ln(1+x) \approx x$, so:
\begin{equation}
    P_{\min} \approx \frac{\gamma \tau^*}{\tau_D} \cdot \frac{\kB T}{2} \cdot \frac{\tau_D}{\tau^*} = \frac{\gamma \kB T}{2}
\end{equation}

\bigskip
\noindent\textit{Manuscript prepared January 2026. Released under CC-BY 4.0.}

\end{document}